\documentclass{scrartcl}  

\usepackage{a4wide}

\usepackage{lineno,hyperref}
\modulolinenumbers[5]

\newcommand{\citep}{\cite}

\usepackage[utf8]{inputenc} 
\usepackage[T1]{fontenc}    

\usepackage{url}            
\usepackage{booktabs}       
\usepackage{amsfonts}       
\usepackage{multirow}
\usepackage{rotating}
\usepackage{amsmath}
\usepackage{amssymb}
\usepackage{amsthm}
\usepackage{enumitem}
\usepackage{tikz}
\usepackage{pgfplots}
\usetikzlibrary{pgfplots.statistics}

\def\LL{\mathrm{LL}}

\def\PL{\mathrm{Pen}}
\def\BIC{\mathrm{BIC}}
\def\X{\mathcal{X}}
\def\Y{\mathcal{X}_1}
\usepackage{mathrsfs}
\DeclareMathOperator*{\argmax}{argmax}

\usepackage{amssymb}
\usepackage{amsfonts}
\newtheorem{Lemma}{Lemma}
\newtheorem{Theorem}{Theorem}
\newtheorem{Corollary}{Corollary}

\let\emptyset\varnothing

\title{\textbf{Entropy-based Pruning for Learning Bayesian Networks using BIC}}

\author{
      \textbf{Cassio P. de Campos}\\
      Queen's University Belfast, UK
      \vspace{0.6em}\\
      \textbf{Mauro Scanagatta}\\
      \textbf{Giorgio Corani}\\
      \textbf{Marco Zaffalon}\\
      Istituto Dalle Molle di studi sull'Intelligenza Artificiale
      (IDSIA)\\ Lugano, Switzerland
      }
\date{}

\begin{document}
\maketitle

\begin{abstract}
\noindent For decomposable score-based structure learning of Bayesian
  networks, existing approaches first compute a collection of
  candidate parent sets for each variable and then optimize over this
  collection by choosing one parent set for each variable without
  creating directed cycles while maximizing the total score. We target
  the task of constructing the collection of candidate parent sets
  when the score of choice is the Bayesian Information Criterion (BIC). We
  provide new non-trivial results that can be used to prune the search
  space of candidate parent sets of each node. We analyze how these
  new results relate to previous ideas in the literature both
  theoretically and empirically. We show in experiments with UCI data
  sets that gains can be significant. Since the new pruning rules are
  easy to implement and have low computational costs, they can be promptly
  integrated into all state-of-the-art methods for structure learning
  of Bayesian networks.
\end{abstract}

\section{Introduction}

A Bayesian network~\citep{pearl1988} is a well-known probabilistic graphical model with applications in a variety of
fields. It is composed of (i) an acyclic directed graph (DAG) where each node is associated to a
random variable and arcs represent dependencies between the variables entailing the {\em Markov}
condition: every variable is conditionally independent of its non-descendant variables given its
parents; and (ii) a set of conditional probability mass functions defined for each variable given
its parents in the graph.  Their graphical nature makes Bayesian networks excellent models for representing the
complex probabilistic relationships existing in many real problems ranging from bioinformatics to
law, from image processing to economic risk analysis.

Learning the structure (that is, the graph) of a Bayesian network from
complete data is an NP-hard task~\citep{chickering2004}.  We are
interested in score-based learning, namely finding the structure which
maximizes a score that depends on the data~\citep{HGC95}. A typical
first step of methods for this purpose is to build a list of suitable
candidate parent sets for each one of the $n$ variables of the
domain. Later an optimization is run to find one element from each
such list in a way that maximizes the total score and does not create
directed cycles.  This work concerns pruning ideas in order to build
those lists. The problem is unlikely to admit a polynomial-time (in $n$)
algorithm, since it is proven to be LOGSNP-hard~\citep{koivisto2006parent}.
Because of that, usually one forces a maximum in-degree (number of
parents per node) $k$ and then simply computes the score of all parent
sets that contain up to $k$ parents. A worth-mention exception is the
greedy search of the K2 algorithm~\citep{cooper1992bayesian}.

A high in-degree implies a large search space for the optimization and thus increases the
possibility of finding better structures. On the other hand, it requires higher computational time, since
there are $\Theta(n^k)$ candidate parent sets for a bound of $k$ if an exhaustive search is
performed.  Our contribution is to provide new rules for pruning sub-optimal parent sets when dealing
with the {\em Bayesian Information Criterion} score~\citep{schwarz1978}, one of the most used score
functions in the literature. We devise new theoretical bounds that can be used in conjunction with
currently published ones~\citep{decampos2011a}. The new results
provide tighter bounds on the maximum number of parents of each
variable in the optimal graph, as well as new pruning techniques that
can be used to skip large portions of the search space without any
loss of optimality. Moreover, the bounds can be efficiently computed
and are easy to implement, so they can be promptly integrated into existing software for learning Bayesian
networks and imply immediate computational gains.

The paper is divided as follows. Section~\ref{s1} presents the problem,
some background and notation. Section~\ref{s2} describes the existing
results in the literature, and Section~\ref{s3} contains the
theoretical developments for the new bounds and pruning rules. Section~\ref{bounds:exp}
shows empirical results comparing the new results against previous ones,
and finally some conclusions are given in Section~\ref{sconc}.

\section{Structure learning of Bayesian networks}\label{s1}

Consider the problem of learning the structure of a Bayesian Network
from a complete data set of $N\geq 2$ instances $\mathcal{D} = \{D_1,
\ldots, D_N\}$. The set of $n\geq 2$ categorical random variables is
denoted by $\X=\{X_1,\ldots, X_n\}$ (each variable has at least two
categories). The state space of $X_i$ is denoted $\Omega_{X_i}$ and a
joint space for $\Y\subseteq \X$ is denoted by $\Omega_{\Y}=\times_{X\in
  \Y}\Omega_X$ (and with a slight abuse $|\Omega_{\emptyset}|=1$
containing a null element). The goal is to find the best DAG $\mathcal{G} = (V, E)$, where $V$
is the collection of nodes (associated one-to-one with the variables in $\X$) and $E$ is the collection of arcs.  $E$ can
be represented by the (possibly empty) set of
parents ${\Pi_1, ..., \Pi_n}$ of each node/variable.

Different score functions can be used to assess the quality of a DAG. This paper regards the {\em
  Bayesian Information Criterion} (or simply $\mathrm{BIC}$)~\citep{schwarz1978}, which
asymptotically approximates the posterior probability of the DAG.  The $\mathrm{BIC}$ score is
\emph{decomposable}, that is, it can be written as a sum of the scores of each variable and its
parent set:
\begin{align*}
& \BIC(\mathcal{G})  = \sum_{i=1}^{n} \BIC(X_i|\Pi_i)=\sum_{i=1}^n
  \left(\LL(X_i|\Pi_i) + \PL(X_i|\Pi_i)\right)\, \\
  \end{align*}
  where $\LL(X_i|\Pi_i)$ denotes the log-likelihood of $X_i$ and its parent set:
\begin{align*}  
& \LL(X_i|\Pi_i) =\displaystyle\sum_{\pi\in\Omega_{\Pi_i}}\sum_{x\in\Omega_{X_i}} N_{x,\pi}\log_b\hat{\theta}_{x|\pi}\, , 
\end{align*}
\noindent
where the base $b\geq 2$ is usually taken as natural or 2. We will make it clear when the result depends on such base.
Moreover, $\hat{\theta}_{x|\pi} $ is the maximum likelihood estimate of the conditional probability
$P(X_i=x|\Pi_i=\pi)$, that is, $N_{x,\pi} / N_{\pi}$; $N_{x,\pi}$
represents the number of times $(X_i=x\land\Pi_i=\pi)$ appears in the
  data set (if $\pi$ is null, then $N_{\pi}=N$ and $N_{x,\pi}=N_x$). In the case with no parents, we use the notation $\LL(X_i)=\LL(X_i|\emptyset)$. $\PL(X_i|\Pi_i)$ is the complexity penalization for $X_i$ and its parent set:
\begin{align*}  
& \PL(X_i|\Pi_i) = -\frac{\log_b N}{2}(|\Omega_{X_i}|-1)(|\Omega_{\Pi_i}|)\, ,
\end{align*}
\noindent 
again with the notation $\PL(X_i)=\PL(X_i|\emptyset)$.

For completeness, we present the definition of (conditional) mutual
information. Let $\X_1$, $\X_2$, $\X_3$ be two-by-two disjoint subsets of $\X$. Then
\[
\mathrm{I}(\X_1,
\X_2|\X_3) = \mathrm{H}(\X_1|\X_3) -
\mathrm{H}(\X_1|\X_2\cup \X_3)
\]
\noindent (unconditional version is obtained with $\X_3=\emptyset$),
and (the sample estimate of) entropy is defined as usual:
$\mathrm{H}(\X_1|\X_2) = \mathrm{H}(\X_1\cup\X_2) -
\mathrm{H}(\X_2)$ and
\[
\mathrm{H}(\X_1) = -\sum_{x\in\Omega_{\X_1}} \frac{N_{x}}{N}\log_b\left( \frac{N_{x}}{N}\right)\, .
\]
\noindent ($x$ runs over the configurations of $\X_1$.) Since $\hat{\theta}_x=N_x/N$, it is clear that
$N\cdot\mathrm{H}(\X_1|\X_2) = -\LL(\X_1|X_2)$ for any disjoint
subsets $\X_1,\X_2\subseteq\X$.

The ultimate goal is to find
$\mathcal{G}^*\in\argmax_{\mathcal{G}}\BIC(\mathcal{G})$ (we avoid
equality because there might be multiple optima). We assume
that if two DAGs $\mathcal{G}_1$ and $\mathcal{G}_2$ have the same score, then we prefer the
graph with fewer arcs. The usual
first step to achieve such goal is the task of finding the {\em candidate parent sets} for a given
variable $X_i$ (obviously a candidate parent set cannot contain $X_i$
itself). This task regards constructing the list 
$L_i$ of parent sets $\Pi_i$ for $X_i$
alongside their scores $\BIC(X_i|\Pi_i)$. Without any restriction, there are $2^{n-1}$ possible
parent sets, since every subset of $\X\setminus\{X_i\}$ is a candidate. Each score
computation costs $\Theta(N\cdot (1+|\Pi_i|))$, and the number of
score computations becomes quickly prohibitive with the increase
of $n$. In order to avoid losing global optimality, we must guarantee
that $L_i$ contains candidate parent sets that cover those in an optimal DAG. For instance, if we apply a bound $k$ on the
number of parents that a variable can have, then the size of
\[
L_i = \{ \langle\Pi_i, \BIC(X_i|\Pi_i)\rangle~|~ |\Pi_i|\leq k\}
\]
\noindent is $\Theta(n^k)$, but we might lose global optimality (this
is the case if any optimal DAG would have more than $k$ parents for
$X_i$). Irrespective of that, this pruning is not enough if $n$ is
large. Bounds greater than $2$ can already become
prohibitive.  For instance, a bound of $k=2$ is adopted in \citep{Bartlett2015} when
dealing with its largest data set (diabetes), which contains 413
variables. One way of circumventing the
problem is to apply pruning rules which allow us to discard/ignore elements
of $L_i$ in such a way that an optimal parent set is never discarded/ignored.

\section{Pruning rules}\label{s2}

The simplest pruning rule one finds in the literature states that if a candidate subset has better
score than a candidate set, then such candidate set can be safely
ignored, since the candidate subset
will never yield directed cycles if the candidate set itself does not
yield cycles~\citep{Teyssier+Koller:UAI05,deCampos2009}.
By safely ignoring/discarding a candidate set we mean that we are still able to
find an optimal DAG (so no accuracy is lost) even if such parent set
is never used. This is formalized as follows.
\begin{Lemma}\label{lemma1} (Theorem 1 in ~\citep{decampos2011a}, but
  also found elsewhere~\citep{Teyssier+Koller:UAI05}.)
  Let $\Pi^*$ be a candidate parent set for the node $X\in\X$. Suppose there exists a parent set
  $\Pi$ such that $\Pi \subset \Pi^* $ and $\BIC(X|\Pi) \geq
  \BIC(X|\Pi^*)$. Then
  $\Pi^*$ can be safely discarded from the list of candidate parent sets of $X$.
\end{Lemma}
This result can be also written in terms of the list of candidate parent sets. In order to find an optimal
DAG for the structure learning problem, it is sufficient to work with
\[
L_i = \{ \langle\Pi_i, \BIC(X_i|\Pi_i)\rangle~|~\forall \Pi'_i\subset\Pi_i: \BIC(X_i|\Pi_i) > \BIC(X_i|\Pi'_i)\}.
\]
Unfortunately there is no way of applying Lemma~\ref{lemma1} without computing the scores of all
candidate sets, and hence it provides no speed up for building the list (it is nevertheless useful for
later optimizations, but that is not the focus of this work).

There are however pruning rules that can reduce the computation time for finding $L_i$ and that are
still safe.

\begin{Lemma}\label{lem2a}
Let $\Pi\subset\Pi'$ be candidate parent sets for
$X\in\mathcal{X}$.
Then $\LL(X|\Pi) \leq \LL(X|\Pi')$, $\mathrm{H}(X|\Pi) \geq \mathrm{H}(X|\Pi')$ and
$\PL(X|\Pi) > \PL(X|\Pi')$.
\end{Lemma}
\begin{proof}
The inequalities follow directly from the definitions of
log-likelihood, entropy and penalization.
\end{proof}

\begin{Lemma}\label{lemma2} (Theorem 4 in~\citep{decampos2011a}.\footnote{There is an imprecision in
 the Theorem 4 of \citep{decampos2011a}, since $t_i$ as defined there does not account for the constant of BIC/AIC
 while in fact it should. In spite of that, their desired result is clear. We present a proof for completeness.})
Let $X\in\X$ be a node with $\Pi \subset  \Pi^*$ two candidate parent sets, such that
$\BIC(X|\Pi) \geq \PL(X|\Pi^*)$. Then $\Pi^*$ and all its supersets
can be safely ignored when building the list of candidate parent sets for $X$.
\end{Lemma}

\begin{proof} Let $\Pi'\supseteq \Pi^*$. By Lemma~\ref{lem2a}, we have $\PL(X|\Pi^*) \geq
  \PL(X|\Pi')$ (equality only if $\Pi^*=\Pi'$). Then
  $\BIC(X|\Pi) \geq \PL(X|\Pi^*)\Rightarrow \BIC(X|\Pi) \geq \PL(X|\Pi') \Rightarrow
  \BIC(X|\Pi) - \BIC(X|\Pi') \geq -\LL(X|\Pi')$, and we have
  $-\LL(X|\Pi') \geq 0$, so Lemma~\ref{lemma1} suffices to conclude the proof.
\end{proof}

Note that $\BIC(X|\Pi) \geq \PL(X|\Pi^*)$ can as well be written as
$\LL(X|\Pi) \geq \PL(X|\Pi^*)-\PL(X|\Pi)$, and if $\Pi^*=\Pi\cup\{Y\}$
for some $Y\notin\Pi$, then it can be written also as $\LL(X|\Pi) \geq (|\Omega_Y|-1)\PL(X|\Pi)$.
The reasoning behind Lemma~\ref{lemma2} is that the maximum
improvement that we can have in $\BIC$ score by inserting new parents
into $\Pi$ would be achieved if $\LL(X|\Pi)$, which is a non-positive value, grew all the way to zero,
since the penalization only gets worse with more parents. If
$\LL(X|\Pi)$ is already close enough to zero, then the loss in the
penalty part cannot be compensated by the gain of likelihood. The
result holds for every superset because both likelihood and penalty
are monotone with respect to increasing the number of parents.

\section{Novel pruning rules}\label{s3}

In this section we devise novel pruning rules by exploiting the
empirical entropy of variables. We later
demonstrate that such rules are useful to ignore elements while
building the list $L_i$ that cannot be ignored by Lemma~\ref{lemma2},
hence tightening the pruning results available in the literature. In
order to achieve our main theorem, we need some intermediate results.

\begin{Lemma}\label{lem2}
Let $\Pi=\Pi'\cup\{Y\}$ for $Y\notin\Pi'$, with $\Pi,\Pi'$ candidate parent sets for
$X\in\mathcal{X}$.
Then $\LL(X|\Pi) - \LL(X|\Pi') \leq N\cdot \min\{\mathrm{H}(X|\Pi'); \mathrm{H}(Y|\Pi')\}$.
\end{Lemma}
\begin{proof}
This comes from simple manipulations and known bounds to the value of conditional mutual information.
\begin{align*}
\LL(X|\Pi) - \LL(X|\Pi') 
&=N\cdot (\mathrm{H}(X|\Pi') - \mathrm{H}(X|\Pi)) \leq
  N\cdot \mathrm{H}(X|\Pi')\, . \\
\LL(X|\Pi) - \LL(X|\Pi')  &=N\cdot \mathrm{I}(X,Y|\Pi') \\
&=N\cdot (\mathrm{H}(Y|\Pi') - \mathrm{H}(Y|\Pi'\cup\{X\})) \leq
  N\cdot\mathrm{H}(Y|\Pi')\, .
\end{align*}
\end{proof}

\begin{Theorem}
Let $X\in\mathcal{X}$, and $\Pi^*$ be
a parent set for $X$. Let $Y\in\X\setminus\Pi^*$ such that
$N\cdot \min\{\mathrm{H}(X|\Pi^*); \mathrm{H}(Y|\Pi^*)\}\leq (1-|\Omega_Y|)\PL(X|\Pi^*)$.
Then the parent set $\Pi=\Pi^*\cup\{Y\}$ and all its supersets can be
safely ignored when building the list of candidate parents sets for $X$.
 \label{thm1}\end{Theorem}
\begin{proof}
We have that
 \begin{align*}
\BIC(X|\Pi) &= \LL(X|\Pi) + \PL(X|\Pi)\\
&\leq \LL(X|\Pi^*) + N\cdot \min\{\mathrm{H}(X|\Pi^*); \mathrm{H}(Y|\Pi^*)\} + \PL(X|\Pi)\\
&\leq \LL(X|\Pi^*) + (1-|\Omega_Y|)\PL(X|\Pi^*) + \PL(X|\Pi)\\
&= \LL(X|\Pi^*) + \PL(X|\Pi^*)-\PL(X|\Pi) + \PL(X|\Pi)
=\BIC(X|\Pi^*)\, .
\end{align*}
First step is the definition of BIC, second step uses Lemma~\ref{lem2}
and third step uses the assumption of this theorem.
Therefore, $\Pi$ can be safely ignored (Lemma~\ref{lemma1}). Now take any $\Pi'\supset\Pi$. Let
$\Pi''=\Pi'\setminus\{Y\}$. It is immediate that
$N\cdot \min\{\mathrm{H}(X); \mathrm{H}(Y)\} \leq (1-|\Omega_Y|)\PL(X|\Pi^*) \Rightarrow N\cdot \min\{\mathrm{H}(X); \mathrm{H}(Y)\} \leq (1-|\Omega_Y|)\PL(X|\Pi'')$, since $\Pi^*\subset\Pi''$
and hence $-\PL(X|\Pi'') > -\PL(X|\Pi^*)$. The theorem follows by the same arguments as before, now applied to $\Pi'$
and $\Pi''$.
\end{proof}

The rationale behind Theorem~\ref{thm1} is that if the data do not
have entropy in amount enough to beat the penalty function, then there
is no reason to continue expanding the parent set candidates.
Theorem \ref{thm1} can be used for pruning the search space of candidate parent sets without having
to compute their BIC scores. However, we must have available the
conditional entropies $\mathrm{H}(X|\Pi^*)$ and
$\mathrm{H}(Y|\Pi^*)$. The former is usually available, since
$-N\cdot\mathrm{H}(X|\Pi^*)=\LL(X|\Pi^*)$, which it is used to
compute $\BIC(X|\Pi^*)$ (and it is natural to assume that such score
has been already computed at the moment Theorem~\ref{thm1} is checked). Actually, this bound amounts exactly to
the previous result in the literature:
\begin{align*}
&N\cdot \mathrm{H}(X|\Pi^*) \leq (1-|\Omega_Y|)\PL(X|\Pi^*) \iff\\
&\LL(X|\Pi^*) \geq \PL(X|\Pi^*\cup\{Y\}) - \PL(X|\Pi^*) \iff\\
&\BIC(X|\Pi^*) \geq \PL(X|\Pi^*\cup\{Y\})\, .
\end{align*}
\noindent By Theorem~\ref{thm1} we know that $\Pi^*\cup\{Y\}$
and any superset can be safely ignored, which is
the very same condition as in Lemma~\ref{lemma2}. The novelty in
Theorem~\ref{thm1} comes from the term $\mathrm{H}(Y|\Pi^*)$. If such
term is already computed (or if it will need to be computed
irrespective of this bound computation, and thus we do not lose time
computing it for this purpose only), then we get (almost) for free
a new manner to prune parent sets. In case this computation of $\mathrm{H}(Y|\Pi^*)$ is not
considered worth, or if we simply want a faster approach to prune
parent sets, we can resort to a more general version of Theorem~\ref{thm1},
as given by Theorem~\ref{thm1b}.

\begin{Theorem}
Let $X\in\mathcal{X}$, and $\Pi^*,\Pi'$ be
parent sets for $X$ with $\Pi'\subseteq\Pi^*$. Let $Y\in\X\setminus\Pi^*$ such that
$N\cdot \min\{\mathrm{H}(X|\Pi'); \mathrm{H}(Y|\Pi')\}\leq (1-|\Omega_Y|)\PL(X|\Pi^*)$.
Then the parent set $\Pi=\Pi^*\cup\{Y\}$ and all its supersets can be
safely ignored when building the list of candidate parents sets for $X$.
 \label{thm1b}
\end{Theorem}
\begin{proof}
It is well-known (see Lemma~\ref{lem2a}) that $\mathrm{H}(X|\Pi^*) \leq \mathrm{H}(X|\Pi')$ and
$\mathrm{H}(Y|\Pi^*) \leq \mathrm{H}(Y|\Pi')$ for any $X$,$Y$,$\Pi'\subseteq\Pi^*$ as
defined in this theorem, so the result follows from Theorem~\ref{thm1}.
\end{proof}

An important property of Theorem~\ref{thm1b} when compared to Theorem~\ref{thm1} is that all
entropy values regard subsets of the current parent set at our own
choice. For instance, we can choose $\Pi'=\emptyset$ and so they
become entropies of single variables, which can be precomputed efficiently in total time
$O(N\cdot n)$. Another option at this point, if we do not want to compute
$\mathrm{H}(Y|\Pi^*)$ and assuming the cache of $Y$ has been already
created, would be to quickly inspect the cache of
$Y$ to find the most suitable subset of $\Pi^*$ to plug into Theorem~\ref{thm1b}.
Moreover, with Theorem~\ref{thm1b},
we can prune the search space of a variable $X$ without
evaluating the likelihood of parent sets for $X$ (just by using the entropies), and so it could be
used to guide the search even before any heavy computation is done.
The main novelty in Theorems~\ref{thm1} and~\ref{thm1b} is to make use of the
(conditional) entropy of $Y$.

This new pruning approach is not trivially achievable by previous existing bounds for BIC.
It is worth noting the relation with previous work. The restriction of Theorem~\ref{thm1b} can be rewritten as:
\begin{align*}
&N\cdot \min\{\mathrm{H}(X|\Pi'); \mathrm{H}(Y|\Pi')\}\leq (1-|\Omega_Y|)\PL(X|\Pi^*) \iff\\
&N\cdot \min\{\mathrm{H}(X|\Pi'); \mathrm{H}(Y|\Pi')\} +\LL(X|\Pi^*) \leq -\PL(X|\Pi^*\cup\{Y\}) +\BIC(X|\Pi^*).
\end{align*}
\noindent
Note that the condition for Lemma~\ref{lemma2} (known from literature) is
exactly $-\PL(X|\Pi^*\cup\{Y\}) +\BIC(X|\Pi^*) \geq 0$.  Hence,
Theorem~\ref{thm1b} will be effective (while the previous rule in
Lemma~\ref{lemma2} will not)
when $-\PL(X|\Pi^*\cup\{Y\}) +\BIC(X|\Pi^*) < 0$, and so
when $N\cdot \min\{\mathrm{H}(X|\Pi'); \mathrm{H}(Y|\Pi')\}
+\LL(X|\Pi^*) < 0$. Intuitively, the new bound of Theorem~\ref{thm1b} might be more useful when the parent set being evaluated is poor
(hence $\LL(X|\Pi^*)$ is low) while the result in Lemma~\ref{lemma2} plays an important role when the parent set
being evaluated is good (and so $\LL(X|\Pi^*)$ is high).
We provide now a numerical example, detailing two real cases from the well-known UCI data set
 \emph{glass}~\cite{Lichman:2013} where only one bound is activated. 
 
 \begin{table}[ht]
\centering
\caption{Pruning rules' examples using UCI data set {\em glass}. $E_1$
  and $E_2$ are two cases of interest. Variable numbers are indexed from 0 to 7
 from left to right in the data table. Name convention follows Theorem~\ref{thm1} to facilitate the understanding.}
\label{t00}
\begin{tabular}{cccccc}
\toprule
          & Target variable ($X$) & $|\Omega_X|$ & $\Pi^* $ & $Y$ & $|\Omega_Y|$  \\ \midrule
$E_1$ & $X_0$       & 2        & $\{X_1, X_2, X_{3}\}$ & $X_7$    & 7       \\
$E_2$ & $X_1$      & 2         & $\{X_0, X_2, X_7\}$     & $X_{6}$ & 2           \\ \toprule

    & $\Pi = \Pi^* \cup\{Y\}$ & $N$ & $N\cdot\mathrm{H}(X|\Pi^*)$ &  $N\cdot\mathrm{H}(Y|\Pi^*)$ & $(1-|\Omega_Y|)\PL(X|\Pi^*)$ \\ \midrule

$E_1$    & $\{X_1, X_2, X_3, X_{7}\}$ & $214$ & $126.68$ & $210.88$ & $128.76$ \\
$E_2$  & $\{X_0, X_2, X_6, X_{7}\}$ & $214$ & $102.81$ & $71.23$ &  $75.11$ \\ \bottomrule
\end{tabular}
\end{table}

Consider case $E_1$ in Table~\ref{t00}: We are constructing the list of candidate parent sets for $X=X_0$, and have just
computed the BIC score of $\Pi^*=\{X_1, X_2, X_{3}\}$. We are interested whether $\Pi=\Pi^*\cup\{X_7\}$ is a good
parent set. We have that 
$126.68 = N\cdot\mathrm{H}(X|\Pi^*)\leq (1-|\Omega_Y|)\PL(X|\Pi^*) = 128.76$, and thus the old pruning rule
(Lemma~\ref{lemma2}) is activated. On the other hand,
$N\cdot\mathrm{H}(Y|\Pi^*)=210.88$, so the new bound that uses the (conditional)
entropy of $Y$ is not activated. 

Now consider case $E_2$ in Table~\ref{t00}: We are building the list of candidate parent sets for $X=X_1$ and have just
computed the BIC score of $\Pi^*=\{X_0, X_2, X_7\}$.  We are interested whether $\Pi=\Pi^*\cup\{X_{6}\}$ is a good
parent set. We have that $71.23=N\cdot\mathrm{H}(Y|\Pi^*) \leq (1-|\Omega_Y|)\PL(X|\Pi^*) = 75.11$, and thus the new bound is activated,
while $N\cdot\mathrm{H}(X|\Pi^*)=102.81$ does not activate the old bound.

The result of Theorem~\ref{thm1b} can also be used to bound the maximum number of parents in any
given candidate parent set. While the asymptotic result is already implied by previous
work~\cite{decampos2011a}, we obtain the finer and interesting result of Theorem~\ref{lem3b}.
 
\begin{Theorem}
There is an optimal structure such that variable $X\in\X$ has at most 
\[
\max_{Y\in\X\setminus\{X\}} \left\lceil 1 + \log_2\left(\frac{\min\{\mathrm{H}(X); \mathrm{H}(Y)\}}{(|\Omega_X|-1)(|\Omega_Y|-1)}\right)+
\log_2 N - \log_2\log_b N\right\rceil^+
\]
parents, where $\lceil\cdot\rceil^+$ denotes the smallest natural
number greater than or equal to its argument.
\label{lem3b}
\end{Theorem}
\begin{proof}
If $\Pi=\emptyset$ is the optimal parent for $X$, then the result
trivially follows since $|\Pi|=0$. Now take
$\Pi$ such that $Y\in\Pi$ and $\Pi^*=\Pi\setminus\{Y\}$. Since $|\Pi|=|\Pi^*|+1$ and $|\Omega_{\Pi^*}|\geq 2^{|\Pi^*|}$,
we have $|\Pi|\leq \log_2 |\Omega_{\Pi^*}| + 1$. Now, if
\begin{align*}
&\log_2 |\Omega_{\Pi^*}| \geq 1 + \log_2\left(\frac{\min\{\mathrm{H}(X); \mathrm{H}(Y)\}}{(|\Omega_X|-1)(|\Omega_Y|-1)}\right)+\log_2 N - \log_2\log_b N\iff\\
&\log_2 |\Omega_{\Pi^*}| \geq \log_2\left(\frac{2 \min\{\mathrm{H}(X); \mathrm{H}(Y)\}}{(|\Omega_X|-1)(|\Omega_Y|-1)}\cdot\frac{N}{\log_b N}\right)\iff\\
&N\cdot \min\{\mathrm{H}(X); \mathrm{H}(Y)\}\leq \frac{\log_b N}{2}\cdot |\Omega_{\Pi^*}|(|\Omega_X|-1)(|\Omega_Y|-1) \iff\\
&N\cdot \min\{\mathrm{H}(X); \mathrm{H}(Y)\} \leq (1-|\Omega_Y|)\cdot
  \PL(X|\Pi^*)\, ,
\end{align*}
\noindent then by Theorem \ref{thm1b} (used with $\Pi'=\emptyset$) every
super set of $\Pi^*$ containing $Y$ can be safely ignored, and so it would
be $\Pi$. Therefore,
\[
|\Pi| \leq 1 + \log_2 |\Omega_{\Pi^*}| < 1 + 1 + \log_2\left(\frac{\min\{\mathrm{H}(X); \mathrm{H}(Y)\}}{(|\Omega_X|-1)(|\Omega_Y|-1)}\right) +\log_2 N -
\log_2\log_b N\, ,
\]
\noindent and since $|\Pi|$ is a natural number, the result follows
by applying the same reasoning for every $Y\in\X\setminus\{X\}$.
\end{proof}

Corollary~\ref{lem3a} is demonstrated for
completeness, since it is implied by previous work (see for instance~\cite{decampos2011a}). It is nevertheless
presented here in more detailed terms and without an asymptotic function.

\begin{Corollary}
There is an optimal structure such that each variable has at most 
$\lceil 1+\log_2 N - \log_2\log_b N\rceil$ parents.
\label{lem3a}
\end{Corollary}
\begin{proof}
By Theorem~\ref{lem3b}, we have that $\Pi$ can be a parent of a node
$X$ only if
\begin{align*}
|\Pi| &\leq \max_{Y\in\X\setminus\{X\}}\left\lceil 1 + \log_2\left(\frac{\min\{\mathrm{H}(X);\mathrm{H}(Y)\}}{(|\Omega_X|-1)(|\Omega_Y|-1)}\right)+\log_2 N - \log_2\log_b N \right\rceil^+\\
 &\leq \max_{Y\in\X\setminus\{X\}}\left\lceil 1 + \log_2\left(\frac{\mathrm{H}(Y)}{(|\Omega_X|-1)(|\Omega_Y|-1)}\right)+\log_2 N - \log_2\log_b N \right\rceil^+\\
&\leq \max_{Y\in\X\setminus\{X\}}\max\left\lceil 1 + \log_2\left(\frac{\log_b |\Omega_Y|}{(|\Omega_X|-1)(|\Omega_Y|-1)}\right)+\log_2 N - \log_2\log_b N \right\rceil^+\\
&\leq  \left\lceil 1 + \log_2\left(\frac{1}{(|\Omega_X|-1)}\right)+\log_2 N - \log_2\log_b N \right\rceil^+\\
&\leq \lceil  1 +\log_2 N - \log_2\log_b N\rceil^+\\
& \leq \lceil 1+\log_2 N - \log_2\log_b N\rceil\, ,
\end{align*}
\noindent since it is assumed that $N\geq 2$ and $b\geq 2$.
\end{proof}

Theorem~\ref{lem3b} can be used to bound the
number of parent sets per variable, even before computing parent sets
for them, with the low computation cost of computing the empirical
entropy of each variable once (hence overall cost of $O(n\cdot N)$
time). We point out that Theorem~\ref{lem3b} can provide effective
bounds (considerably smaller than $\lceil 1+\log_2 N - \log_2\log_b N\rceil$) on the number of parents for specific variables, particularly
when number of states is high and entropies are low, as we will see in
the next section.

\section{Experiments}
\label{bounds:exp}

We run experiments using a collection of data sets from the UCI
repository~\cite{Lichman:2013}. Table ~\ref{t1} shows the data set
names, number of variables $n$ and number of data points $N$. In the
same table, we show the maximum number of parents that a node can
have, according to the new result of Theorem~\ref{lem3b}, as well as
the old result from the literature (which we present in
Corollary~\ref{lem3a}). The old bound is global, so a single
number is given in column 5, while the new result of
Theorem~\ref{lem3b} implies a different maximum number of parents per
node. We use the notation {\em bound (number of times)}, with the
bound followed by the number of nodes for which the new bound reached
that value, in parenthesis (so all numbers in parenthesis in a row
should sum to $n$ of that row).
We see that the gains with the new bounds are quite significant and can
prune great parts of the search space further than previous results.
\begin{table}[ht]
\centering
\begin{tabular}{l  c  c | l  c }  \toprule
& & & \multicolumn{2}{c}{Bound on number of parents} \\
Dataset 	& $n$ & $N$  	& Theorem 3 & Corollary 1 \\ \midrule
          glass &     8 &   214 &                             6 (7), 3 (1) &     6 \\
       diabetes &     9 &   768 &                                    7 (9) &     8 \\
    tic-tac-toe &    10 &   958 &                                   6 (10) &     8 \\
            cmc &    10 &  1473 &                             8 (3), 7 (7) &     9 \\
  breast-cancer &    10 &   286 &               6 (4), 5 (2), 4 (1), 3 (3) &     7 \\
  solar-flare &    12 &  1066 &        7 (4), 6 (1), 5 (5), 3 (1), 2 (1) &     8 \\
        heart-h &    12 &   294 &               6 (6), 5 (3), 4 (2), 3 (1) &     7 \\
          vowel &    14 &   990 &                            8 (12), 4 (2) &     8 \\
            zoo &    17 &   101 &                     5 (10), 4 (6), 2 (1) &     5 \\
           vote &    17 &   435 &                            7 (15), 6 (2) &     7 \\
        segment &    17 &  2310 &                            9 (16), 6 (1) &     9 \\
          lymph &    18 &   148 &                      5 (8), 4 (8), 3 (2) &     6 \\
  primary-tumor &    18 &   339 &               6 (9), 5 (7), 4 (1), 2 (1) &     7 \\
        vehicle &    19 &   846 &                            7 (18), 6 (1) &     8 \\
      hepatitis &    20 &   155 &                            5 (18), 4 (2) &     6 \\
          colic &    23 &   368 &                     6 (8), 5 (12), 4 (3) &     7 \\
          autos &    26 &   205 &       6 (16), 5 (3), 4 (1), 3 (5), 1 (1) &     6 \\
          flags &    29 &   194 &        6 (5), 5 (7), 4 (7), 3 (7), 2
                                  (3) &     6 \\
\bottomrule
\end{tabular}
\caption{Maximum number of parents that nodes have using new (column
  4) and previous bounds (column 5). In column 4, we list the bound on number
  of parents followed by how many nodes have that bound in parenthesis
  (the new theoretical results obtain a specific bound per node,
  while previous results obtain a single global bound).}\label{t1}
\end{table}

Our second set of experiments compares the activation of
Theorems~\ref{thm1},~\ref{thm1b}, and~\ref{lem3b} in pruning the search space for
the construction of the list of candidate parent sets. Tables~\ref{t2}
and~\ref{t3} (in the end of this document) present the results as follows. Columns one to four
contain, respectively, the data set name, number of variables, number of data points
and maximum in-degree (in-d) that we impose (a maximum in-degree is
impose so as we can compare the obtained results among different
approaches). The fifth column, named $|S|$, presents the total number
of parent sets that need to be evaluated by the brute-force procedure
(taking into consideration the imposed maximum in-degree). Columns 6
to 12 present the number of times that different pruning results are
activated when exploring the whole search space. Larger numbers means
that more parent sets are ignored (even without being evaluated). The
naming convention for the pruning algorithms as used on those columns is:
\begin{enumerate}[label={{\em Alg}\arabic*}]
\item Application of Theorem~\ref{thm1} using $\mathrm{H}(X|\Pi^*)$ in the
  expression of the rule (instead of the minimization), where $X$ is the variable for which we are
  building the list and
  $\Pi$ is the current parent set being explored. This is equivalent to the previous rule in
  the literature, as presented in this paper in Lemma~\ref{lemma2}.
\item Application of Theorem~\ref{thm1} using $\mathrm{H}(Y|\Pi^*)$ in the
  expression of the rule (instead of the minimization), where $X$ is the variable for which we are
  building the list and
  $Y$ is the variable just to be inserted in the parent set $\Pi^*$
  that is being
  explored. This is the new pruning rule which makes most use of
  entropy, but it may be slower than the others (since conditional
  entropies might need to be evaluated, if they were not yet).
\item  Application of Theorem~\ref{thm1b} using $H(X)$ in the formula, that is, with
  $\Pi'=\emptyset$ (and instead of the
  minimization). This is a slight improvement to the known rule in the
  literature regarding the maximum number of parents of a variable and is very
  fast, since it does not depend on evaluating any parent sets.
\item  Application of Theorem~\ref{thm1b} using $H(Y)$ in the formula, that is, with
  $\Pi'=\emptyset$ (and instead of the minimization). This is a
  different improvement to the known rule in the
  literature regarding the maximum number of parents of a variable and is very
  fast, since it does not depend on evaluating any parent sets.
\end{enumerate}
We also present the combined number of pruning obtained by some of
these ideas when they are applied together. Of particular interest is
column 8 with ({\em Alg}1)+({\em Alg}2), as it shows the largest amount of pruning
that is possible, albeit more computationally costly
because of the (possibly required) computations for ({\em Alg}2). This
is also presented graphically in the boxplot of Figure~\ref{fig111},
where the values for the 18 data sets are summarized and the amount of
pruning is divided by the pruning of ({\em Alg}1), and so a ratio
above one shows (proportional) gain with respect to the previous
literature pruning rule.
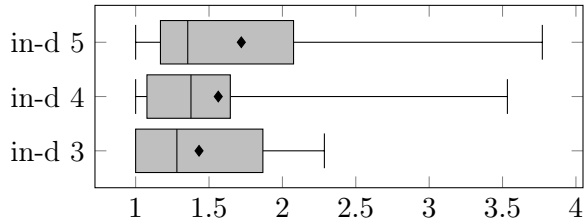
\begin{figure}
\begin{center}
\begin{tikzpicture}
  \begin{axis}
    [
cycle list={{black}},
     every average/.style={/tikz/mark=diamond*},
  boxplot/draw direction=x,
  boxplot/variable width,
  height=4cm,
  width=8cm,
  boxplot/every box/.style={fill=gray!50},
    ytick={1,2,3},
    yticklabels={ in-d 3, in-d 4, in-d 5},
    ]
    \addplot+[mark=diamond*,
    boxplot prepared={
      median=1.281,
      upper quartile=1.867,
      lower quartile=1,
      upper whisker=2.286,
      lower whisker=1
    },
   ]
 coordinates {(0,1.432)};
    \addplot+[mark=diamond*,
    boxplot prepared={
      median=1.377,
      upper quartile=1.645,
      lower quartile=1.077,
      upper whisker=3.533,
      lower whisker=1
    },
    ] coordinates {(0,1.563)};
    \addplot+[mark=diamond*,
    boxplot prepared={
      median=1.355,
      upper quartile=2.076,
      lower quartile=1.169,
      upper whisker=3.771,
      lower whisker=1
    },
    ] coordinates {(0,1.720)};
\end{axis}
\end{tikzpicture}
\end{center}
\caption{Ratio between pruned candidates using ({\em Alg}1) (which
 theoretically subsumes ({\em Alg}3)) and ({\em
    Alg}2) divided by pruned candidates using prune approach
  ({\em Alg}1) alone, for different values of maximum in-degree. Greater than one
  means better than ({\em Alg}1). Results over 18 data sets. Averages
  are marked with a diamond.}\label{fig111}
\end{figure}

Column 12 of Tables~\ref{t2} and~\ref{t3} have the pruning results
(number of ignored candidates) for ({\em Alg}1) and ({\em Alg}4) together, since this represents the pruning obtained
by the old rule plus the new rule give by Theorem~\ref{thm1b} such that no extra computational
cost takes place (and moreover it subsumes approach ({\em Alg}3), since ({\em Alg}1) is theoretically superior to
({\em Alg}3)). Again, this is summarized in the boxplot of
Figure~\ref{fig112} over the 18 data sets and the values are divided
by the amount of pruning of ({\em Alg}1) alone, so values
above one show the (proportional) gain with respect to the previous
literature rule. 

As we can see in more detail in Tables~\ref{t2} and~\ref{t3}, the gains with
the new pruning ideas are significant in many circumstances. Moreover,
there is no extra computational cost for applying
({\em Alg}3) and ({\em Alg}4), so one should always apply those rules
while deciding selectively whether to employ prune ({\em Alg}2) or not
(we recall that one can tune that rule by exploiting the flexibility
of Theorem~\ref{thm1b} and searching for a subset that is
already available in the computed lists, so a more sophisticated pruning scheme
is also possible).
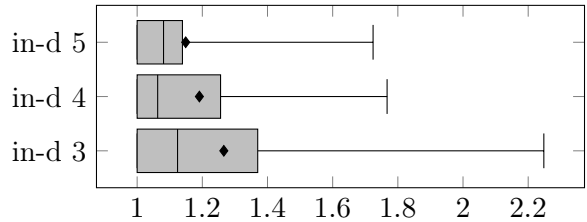
\begin{figure}
\begin{center}
\begin{tikzpicture}
  \begin{axis}
    [
cycle list={{black}},
     every average/.style={/tikz/mark=diamond*},
  boxplot/draw direction=x,
  boxplot/variable width,
  width=8cm,
  height=4cm,
  boxplot/every box/.style={fill=gray!50},
    ytick={1,2,3},
    yticklabels={
     in-d 3, in-d 4, in-d 5}, 
    ]
    \addplot+[mark=diamond*,
    boxplot prepared={
      median=1.124,
      upper quartile=1.370,
      lower quartile=1,
      upper whisker=2.248,
      lower whisker=1
    },
    ] coordinates {(0,1.266)};
    \addplot+[mark=diamond*,
    boxplot prepared={
      median=1.063,
      upper quartile=1.256,
      lower quartile=1,
      upper whisker=1.767,
      lower whisker=1
    },
    ] coordinates {(0,1.191)};
    \addplot+[mark=diamond*,
    boxplot prepared={
      median=1.081,
      upper quartile=1.139,
      lower quartile=1,
      upper whisker=1.724,
      lower whisker=1
    },
    ] coordinates {(0,1.149)};
  \end{axis}
\end{tikzpicture}
\end{center}
\caption{Ratio between pruned candidates using ({\em Alg}1) (which
 theoretically subsumes ({\em Alg}3)) and ({\em
    Alg}4) divided by pruned candidates using prune approach
  ({\em Alg}1) alone, for different values of maximum in-degree. Greater than one
  means better than ({\em Alg}1). Results over 18 data sets. Averages
  are marked with a diamond.}\label{fig112}
\end{figure}

\section{Conclusions}
\label{sconc}
This paper present new non-trivial pruning rules to be used with the
Bayesian Information Criterion (BIC) score for learning the structure
of Bayesian networks. The derived theoretical bounds extend previous
results in the literature and can be promptly integrated into existing
solvers with minimal effort and computational costs.  They imply
faster computations without losing optimality.  The very
computationally efficient version of the new rules imply gains of
around 20\% with respect to previous work, according to our
experiments, while the most computationally demanding pruning achieves
around 50\% more pruning than before.  We conjecture that further
bounds for the BIC score are unlikely to exist unless for some
particular cases and situations.

\section*{Acknowledgments}
 Work partially supported by the Swiss NSF grant n.~200021\_146606~/1
 and ns.~IZKSZ2\_162188.

\bibliographystyle{elsarticle-num}

\newpage

\begin{sidewaystable}
\centering
\begin{tabular}{l  c  c  c  c  c r  r  r r r r r }  \toprule
Dataset  & $n$	  & $N$  & in-d	& $|S|$ & ({\em Alg}1) & ({\em Alg}2) & ({\em Alg}1) + ({\em Alg}2) & ({\em Alg}3) & ({\em Alg}4) & ({\em Alg}3) + ({\em Alg}4) & ({\em Alg}1) + ({\em Alg}4) \\ \midrule
\multirow{3}{*}{glass} & \multirow{3}{*}{8} & \multirow{3}{*}{214} 
& 3 & 504 & 0 & 0 & 0 & 0 & 0 & 0 & 0 \\
& & & 4 & 784 & 114 & 66 & 154 & 0 & 0 & 0 & 114 \\
& & & 5 & 952 & 222 & 171 & 280 & 105 & 57 & 126 & 240 \\
\multirow{3}{*}{diabetes} & \multirow{3}{*}{9} & \multirow{3}{*}{768} 
& 3 & 828 & 0 & 0 & 0 & 0 & 0 & 0 & 0 \\
& & & 4 & 1458 & 0 & 0 & 0 & 0 & 0 & 0 & 0 \\
& & & 5 & 1962 & 0 & 0 & 0 & 0 & 0 & 0 & 0 \\
\multirow{3}{*}{tic-tac-toe} & \multirow{3}{*}{10} & \multirow{3}{*}{958} 
& 3 & 1290 & 0 & 0 & 0 & 0 & 0 & 0 & 0 \\
& & & 4 & 2550 & 0 & 0 & 0 & 0 & 0 & 0 & 0 \\
& & & 5 & 3810 & 659 & 1114 & 1244 & 504 & 504 & 504 & 659 \\
\multirow{3}{*}{breast-cancer} & \multirow{3}{*}{10} & \multirow{3}{*}{286} 
& 3 & 1290 & 624 & 578 & 704 & 611 & 528 & 671 & 684 \\
& & & 4 & 2550 & 1776 & 1660 & 1902 & 1756 & 1475 & 1855 & 1874 \\
& & & 5 & 3810 & 3021 & 2905 & 3161 & 2997 & 2589 & 3109 & 3130 \\
\multirow{3}{*}{cmc} & \multirow{3}{*}{10} & \multirow{3}{*}{1473} 
& 3 & 1290 & 0 & 0 & 0 & 0 & 0 & 0 & 0 \\
& & & 4 & 2550 & 30 & 81 & 106 & 7 & 27 & 34 & 53 \\
& & & 5 & 3810 & 410 & 691 & 877 & 271 & 483 & 660 & 707 \\
\multirow{3}{*}{heart-h} & \multirow{3}{*}{12} & \multirow{3}{*}{294} 
& 3 & 2772 & 206 & 356 & 471 & 196 & 330 & 459 & 463 \\
& & & 4 & 6732 & 1820 & 2710 & 3098 & 1615 & 2384 & 2908 & 2942 \\
& & & 5 & 12276 & 6223 & 7914 & 8449 & 5837 & 6941 & 8110 & 8157 \\
\multirow{3}{*}{solar-flare} & \multirow{3}{*}{12} & \multirow{3}{*}{1066} 
& 3 & 2772 & 825 & 1326 & 1666 & 687 & 1131 & 1468 & 1544 \\
& & & 4 & 6732 & 3419 & 4772 & 5464 & 2945 & 4170 & 5068 & 5211 \\
& & & 5 & 12276 & 8144 & 10197 & 11005 & 7324 & 9137 & 10563 & 10726 \\
\multirow{3}{*}{vowel} & \multirow{3}{*}{14} & \multirow{3}{*}{990} 
& 3 & 5278 & 288 & 478 & 614 & 132 & 111 & 243 & 379 \\
& & & 4 & 15288 & 1718 & 2468 & 2854 & 1232 & 551 & 1343 & 1809 \\
& & & 5 & 33306 & 10257 & 13170 & 14247 & 8162 & 2333 & 9065 & 11140 \\
\multirow{3}{*}{zoo} & \multirow{3}{*}{17} & \multirow{3}{*}{101} 
& 3 & 11832 & 1704 & 2354 & 2655 & 735 & 518 & 1237 & 2189 \\
& & & 4 & 42772 & 12652 & 17034 & 20055 & 7875 & 2843 & 10201 & 14807 \\
& & & 5 & 117028 & 70383 & 91180 & 94306 & 32541 & 20162 & 49366 & 79220 \\
\multirow{3}{*}{segment} & \multirow{3}{*}{17} & \multirow{3}{*}{2310} 
& 3 & 11832 & 0 & 0 & 0 & 0 & 0 & 0 & 0 \\
& & & 4 & 42772 & 201 & 309 & 506 & 0 & 0 & 0 & 201 \\
& & & 5 & 117028 & 3983 & 7580 & 10635 & 0 & 0 & 0 & 3983 \\
\end{tabular}
\caption{Pruning results for multiple UCI data sets. Columns contain,
  respectively: data set name, number of variables, number of data
  points, maximum imposed in-degree, size of search space, followed by
  the number of pruned parent sets when considering (a combination
  of) different pruning rules (see the list of pruning rules for more details). }\label{t2}
\end{sidewaystable}

\begin{sidewaystable}
\centering
\begin{tabular}{l  c  c  c  c  c r  r  r r r r r }  \toprule
Dataset  & $n$	  & $N$  & in-d	& $|S|$ & ({\em Alg}1) & ({\em Alg}2) & ({\em Alg}1) + ({\em Alg}2) & ({\em Alg}3) & ({\em Alg}4) & ({\em Alg}3) + ({\em Alg}4) & ({\em Alg}1) + ({\em Alg}4) \\ \midrule
\multirow{3}{*}{vote} & \multirow{3}{*}{17} & \multirow{3}{*}{435} 
& 3 & 11832 & 0 & 0 & 0 & 0 & 0 & 0 & 0 \\
& & & 4 & 42772 & 0 & 0 & 0 & 0 & 0 & 0 & 0 \\
& & & 5 & 117028 & 577 & 852 & 1429 & 0 & 0 & 0 & 577 \\
\multirow{3}{*}{primary-tumor} & \multirow{3}{*}{18} & \multirow{3}{*}{339} 
& 3 & 14994 & 2196 & 2634 & 3097 & 2176 & 1475 & 3034 & 3049 \\
& & & 4 & 57834 & 13159 & 18299 & 20780 & 12827 & 11817 & 20073 & 20370 \\
& & & 5 & 169218 & 61443 & 98794 & 109665 & 58182 & 65331 & 98280 & 99988 \\
\multirow{3}{*}{lymph} & \multirow{3}{*}{18} & \multirow{3}{*}{148} 
& 3 & 14994 & 3254 & 4100 & 5779 & 2822 & 2921 & 4689 & 5066 \\
& & & 4 & 57834 & 31751 & 37285 & 44532 & 27057 & 23232 & 37284 & 40225 \\
& & & 5 & 169218 & 130426 & 147311 & 155807 & 118134 & 97946 & 143188 & 149137 \\
\multirow{3}{*}{vehicle} & \multirow{3}{*}{19} & \multirow{3}{*}{846} 
& 3 & 18753 & 3 & 3 & 6 & 0 & 0 & 0 & 3 \\
& & & 4 & 76893 & 697 & 1096 & 1750 & 0 & 0 & 0 & 697 \\
& & & 5 & 239685 & 10863 & 23197 & 32306 & 0 & 0 & 0 & 10863 \\
\multirow{3}{*}{hepatitis} & \multirow{3}{*}{20} & \multirow{3}{*}{155} 
& 3 & 23180 & 0 & 0 & 0 & 0 & 0 & 0 & 0 \\
& & & 4 & 100700 & 0 & 0 & 0 & 0 & 0 & 0 & 0 \\
& & & 5 & 333260 & 21692 & 64380 & 81809 & 0 & 0 & 0 & 21692 \\
\multirow{3}{*}{colic} & \multirow{3}{*}{23} & \multirow{3}{*}{368} 
& 3 & 41239 & 3108 & 4622 & 5896 & 2362 & 3515 & 4536 & 5067 \\
& & & 4 & 209484 & 94741 & 98958 & 123967 & 84182 & 79540 & 109266 & 115882 \\
& & & 5 & 815166 & 640892 & 643377 & 709743 & 607555 & 515894 & 676558 & 692531 \\
\multirow{3}{*}{autos} & \multirow{3}{*}{26} & \multirow{3}{*}{205} 
& 3 & 68250 & 20393 & 22139 & 27655 & 17625 & 14600 & 22815 & 25223 \\
& & & 4 & 397150 & 237949 & 252244 & 292005 & 186808 & 140784 & 232158 & 272899 \\
& & & 5 & 1778530 & 1384004 & 1523500 & 1605611 & 1221799 & 850889 & 1448294 & 1533419 \\
\multirow{3}{*}{flags} & \multirow{3}{*}{29} & \multirow{3}{*}{194} 
& 3 & 106720 & 61999 & 57143 & 74774 & 59060 & 46597 & 69884 & 71439 \\
& & & 4 & 700495 & 579352 & 547529 & 634327 & 562905 & 405275 & 613594 & 620955 \\
& & & 5 & 3550586 & 3323460 & 3301446 & 3469073 & 3261508 & 2478191 & 3411111 & 3433015 \\
\end{tabular}
\caption{Pruning results for multiple UCI data sets. Columns contain,
  respectively: data set name, number of variables, number of data
  points, maximum imposed in-degree, size of search space, followed by
  the number of pruned parent sets when considering (a combination
  of) different pruning rules (see the list of pruning rules for more details). }\label{t3}
\end{sidewaystable}

\end{document}